\documentclass[conference]{IEEEtran}
\usepackage{times}
\usepackage{amsmath,amssymb,amsthm}
\usepackage{bm}
\usepackage{graphicx}
\usepackage{algorithm}
\usepackage{algpseudocode}
\usepackage{float}
\usepackage{tikz}
\usepackage{booktabs}
\usepackage{eso-pic}   % \AddToShipoutPictureBG*, \AtPageUpperLeft
\usepackage{graphicx}  % \raisebox
\usepackage[table]{xcolor}

\usepackage{graphicx}
\usepackage{subcaption}
\usetikzlibrary{arrows.meta, positioning}

\usepackage[numbers]{natbib}
\usepackage{multicol}
\usepackage[bookmarks=true,hypertexnames=false]{hyperref}

\theoremstyle{plain}

\newtheorem{proposition}{Proposition}

\newtheorem{theorem}{Theorem}
\newtheorem{corollary}{Corollary}
\newtheorem{definition}{Definition}
\newtheorem{remark}{Remark}

\begin{document}

\AddToShipoutPictureBG*{%
  \AtPageUpperLeft{%
    \hspace{0.5\paperwidth}%
    \raisebox{-1.1cm}{%
      \makebox[0pt][c]{%
        \parbox{\textwidth}{%
          \centering
          \normalsize
          \textit{This paper has been accepted for presentation at the RSS 2026 Workshop on \\ From Perception to Action: Representation-Centric Robot Autonomy.}
        }%
      }%
    }%
  }%
}

\title{Self-supervised Geometry Reasoning for LiDAR Simultaneous Localization and Mapping}

\author{
Jiwoo Kim$^{1,*}$\quad
Jinwoo Lee$^{1,*}$\quad
Woojae Shin$^{1}$\quad
Giseop Kim$^{2}$\quad
Hyondong Oh$^{1,\dagger}$ \\[1mm]
{\small $^{1}$Korea Advanced Institute of Science and Technology (KAIST), Daejeon, Republic of Korea} \\
{\small $^{2}$Daegu Gyeongbuk Institute of Science and Technology (DGIST), Daegu, Republic of Korea} \\[1mm]
{\small $^{*}$Equal contribution. $^{\dagger}$Corresponding author.} \\[1mm]
{\small Emails: \texttt{\{tars0523, jinwoolee, oj7987, h.oh\}@kaist.ac.kr}, \texttt{gsk@dgist.ac.kr}}
}
\maketitle

\begin{abstract}
LiDAR simultaneous localization and mapping (SLAM) relies on local geometric quantities such as covariances, correspondences, and surface structures. However, most existing pipelines rely on hand-crafted estimates of local geometry and use them as fixed inputs to LiDAR SLAM, which can make the estimated local geometry noisy and unstable in sparse regions of a point cloud or when using low-resolution LiDAR. To address this issue, this paper introduces a self-supervised framework that learns an explicit symbolic representation of local geometry and uses it to improve LiDAR SLAM recursively. Specifically, each point is represented as a Gaussian distribution, allowing local geometry to be described by a covariance. Without dense geometry labels or ground-truth poses, the framework learns by maximizing the likelihood of local geometry, with self-supervision derived from consistency relations over symbolic geometric representations, including predicted covariances, correspondences, and trajectory from SLAM. The learned geometry is then fed back into LiDAR SLAM, forming a reciprocal loop in which improved geometry enhances localization and mapping, and improved localization provides cleaner supervision for subsequent geometry reasoning. This framework is backend-agnostic and can be plugged into existing LiDAR SLAM pipelines without architectural changes. Experiments on KITTI under varying LiDAR resolutions show that the proposed method improves both odometry and global registration.
\end{abstract}

\IEEEpeerreviewmaketitle

\section{Introduction}
LiDAR simultaneous localization and mapping (SLAM) is a fundamental technology that enables robots to explore, localize, and interact with previously unknown environments~\citep{cadena2017past}. A common principle behind many LiDAR SLAM modules is the use of local geometric information extracted from point clouds. For example, LiDAR odometry often relies on local covariance structures to achieve accurate scan matching~\citep{zhang2014loam}, while global registration methods exploit surface normals and local geometric consistency to establish reliable correspondences across distant scans~\citep{lim2024quatro++}.

However, most existing pipelines estimate such local geometry using hand-crafted neighborhood statistics and then use it as a fixed input to subsequent optimization. These estimates can become noisy and unstable in sparse regions of a point cloud or under low-resolution LiDAR observations, making the overall SLAM performance highly dependent on the environment and sensor characteristics. 

% Deep learning-based approaches have been explored to address these limitations by predicting richer geometric information from point clouds. However, they often require dense geometric supervision, such as complete point clouds or surface annotations, or accurate ground-truth trajectories for training. Such requirements limit their applicability in real-world SLAM scenarios, where dense geometric labels and precise pose annotations are difficult to obtain.

\begin{figure}
    \centering
    \includegraphics[width=0.99\linewidth]{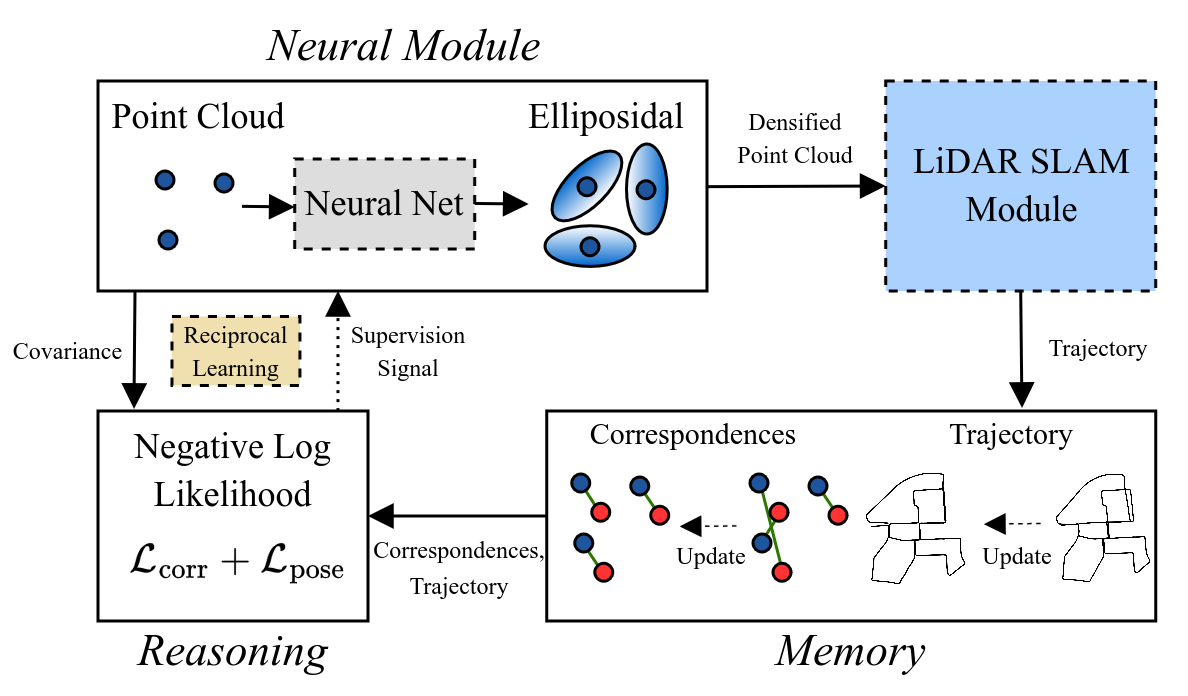}
\caption{Overview of the proposed self-supervised geometry reasoning framework. The neural module infers point-wise covariance as local geometry and samples additional points with this covariance to produce a densified point cloud for LiDAR SLAM. The reasoning module provides likelihood-based self-supervision from correspondences and trajectory in memory, while the updated SLAM trajectory is written back to memory, forming a reciprocal loop between geometry reasoning and SLAM.}
\label{fig:framework}
\end{figure}

In this paper, we move beyond this perspective and propose a self-supervised geometric reasoning framework that recursively infers scene geometry from raw LiDAR observations and leverages it for localization and mapping, without requiring dense geometry labels or ground-truth trajectories. Our formulation is inspired by imperative learning~\citep{wang2024imperative}, which introduces a general self-supervised neuro-symbolic framework where neural module, symbolic reasoning, and memory are reciprocally optimized through bilevel optimization. In our formulation, the neural module predicts local geometry from raw LiDAR observations, while LiDAR SLAM uses this information to estimate the trajectory. The resulting trajectory, in turn, serves as an evolving memory, providing more reliable poses and correspondences for the next round of geometry reasoning. This reciprocal interaction between the neural module, reasoning module, memory, and SLAM backend is summarized in Fig.~\ref{fig:framework}. The proposed framework is backend-agnostic, enabling learned geometric priors to improve existing LiDAR SLAM pipelines in a self-supervised manner.

%%%%%%%%%%%%%%%%%%%%%%%%%%%%%%%%%%%%%%%%%%%%%%%%%%%%%%%%%%%%%%%%%%%%%%%%%%%%%%%%%%%%%%%%%%%%%%%%%%%%%%%%%%%%%%%%%%%%%%%%%%%%%%%%%%%%%%%%%%%%%%%%%%%%%%%%%%%%%%%%%%%%%%%%%%%%%%%%%%%%%%%%%%%%%%%%%%%%%%%%%%%%%%%%%%%%%%%%%%%%%%%%%%%%%%%%%%%%%%%%%%%%%%%%%%%%%%%%%%%%%%%%%%%%%%%%%%%%%%%%%%%%%%%%%%%%%%%%%%%%%%%%%%%%%%%%%%%%%%%%%%%%%%%%%%%%%%%%%%%%%%%%%%%%%%%%%%%%%%%%%%%%%%%%%%%%%%%%%%%%%%%%%%%%%%%%%%%%%%%%%%%%%%%%%%%%%%%%%%%%%%%%%%%%%%%%%%%%%%%%%%%%%%%%%%%%%%%%%%%%%%%%%%%%%%%%%%%%%%%%%%%%%%%%%%%%%%%%%%%%%%%%%%%%%%%%

\section{Related Works}

\subsection{Underlying Geometry Estimation from Point Clouds}

Estimating the underlying geometry of point clouds is essential for geometric perception~\citep{guo2020deep}. Classical methods typically use hand-crafted neighborhood statistics, such as PCA-based covariance analysis~\citep{pauly2003multi} and surface normal estimation~\citep{mitra2003estimating}. Although efficient, these methods heavily depend on local sampling quality and neighborhood selection, making them fragile under sparse, noisy, and non-uniform LiDAR measurements~\citep{lin2014eigen}.

Learning-based methods have been proposed to infer richer geometric structures, often by densifying sparse point clouds or reconstructing complete surfaces~\citep{vizzo2022make, yang2024tulip}. However, these approaches usually require dense ground-truth geometry for supervision and may hallucinate unobserved structures that are not supported by the actual measurements~\citep{yang2024tulip}. Unlike full-scene completion, our work aims to infer local underlying geometry around observed points, providing a compact, uncertainty-aware, and physically grounded representation for geometric perception.

\subsection{Self-supervised Neuro-Symbolic Learning}

Neuro-symbolic learning combines neural perception with symbolic reasoning to improve interpretability and generalization~\citep{wang2024imperative}. In robotics, symbolic reasoning can include not only logical rules~\citep{wang2024imperative} but also physical principles and geometric optimization, such as bundle adjustment~\citep{zhan2024imatching} and pose graph optimization~\citep{fu2024islam}. Recent imperative learning frameworks~\citep{wang2024imperative} use such reasoning engines as sources of self-supervision, allowing neural modules to learn without explicit labels.

Our work instantiates imperative learning~\citep{wang2024imperative} for point cloud geometry reasoning in LiDAR SLAM. We specialize the reasoning component as a maximum likelihood estimation (MLE) framework for local geometry, where point-wise covariances are learned to explain two types of SLAM consistency: the residuals between corresponding points and the relative poses estimated by the SLAM backend. In this formulation, the neural module predicts Gaussian local geometry, the memory stores SLAM-estimated poses and correspondences, and the likelihoods provide self-supervision by enforcing consistency between them. The resulting local geometry is then fed back into the SLAM pipeline to improve downstream tasks such as odometry and global registration, and the refined SLAM outputs in turn provide cleaner memory for subsequent geometry reasoning. This forms a reciprocal structure in which local geometry and SLAM consistency progressively reinforce each other.

%%%%%%%%%%%%%%%%%%%%%%%%%%%%%%%%%%%%%%%%%%%%%%%%%%%%%%%%%%%%%%%%%%%%%%%%%%%%%%%%%%%%%%%%%%%%%%%%%%%%%%%%%%%%%%%%%%%%%%%%%%%%%%%%%%%%%%%%%%%%%%%%%%%%%%%%%%%%%%%%%%%%%%%%%%%%%%%%%%%%%%%%%%%%%%%%%%%%%%%%%%%%%%%%%%%%%%%%%%%%%%%%%%%%%%%%%%%%%%%%%%%%%%%%%%%%%%%%%%%%%%%%%%%%%%%%%%%%%%%%%%%%%%%%%%%%%%%%%%%%%%%%%%%%%%%%%%%%%%%%%%%%%%%%%%%%%%%%%%%%%%%%%%%%%%%%%%%%%%%%%%%%%%%%%%%%%%%%%%%%%%%%%%%%%%%%%%%%%%%%%%%%%%%%%%%%%%%%%%%%%%%%%%%%%%%%%%%%%%%%%%%%%%%%%%%%%%%%%%%%%%%%%%%%%%%%%%%%%%%%%%%%%%%%%%%%%%%%%%%%%%%%%%%%%%%%

\section{Probabilistic Modeling of Underlying Geometry as a Statistical Manifold}
\label{sec:3}

To formalize this geometry reasoning, we first introduce a probabilistic representation of the underlying geometry~\citep{lee2026learning}.

\begin{figure}
    \centering
    \includegraphics[width=0.9\linewidth]{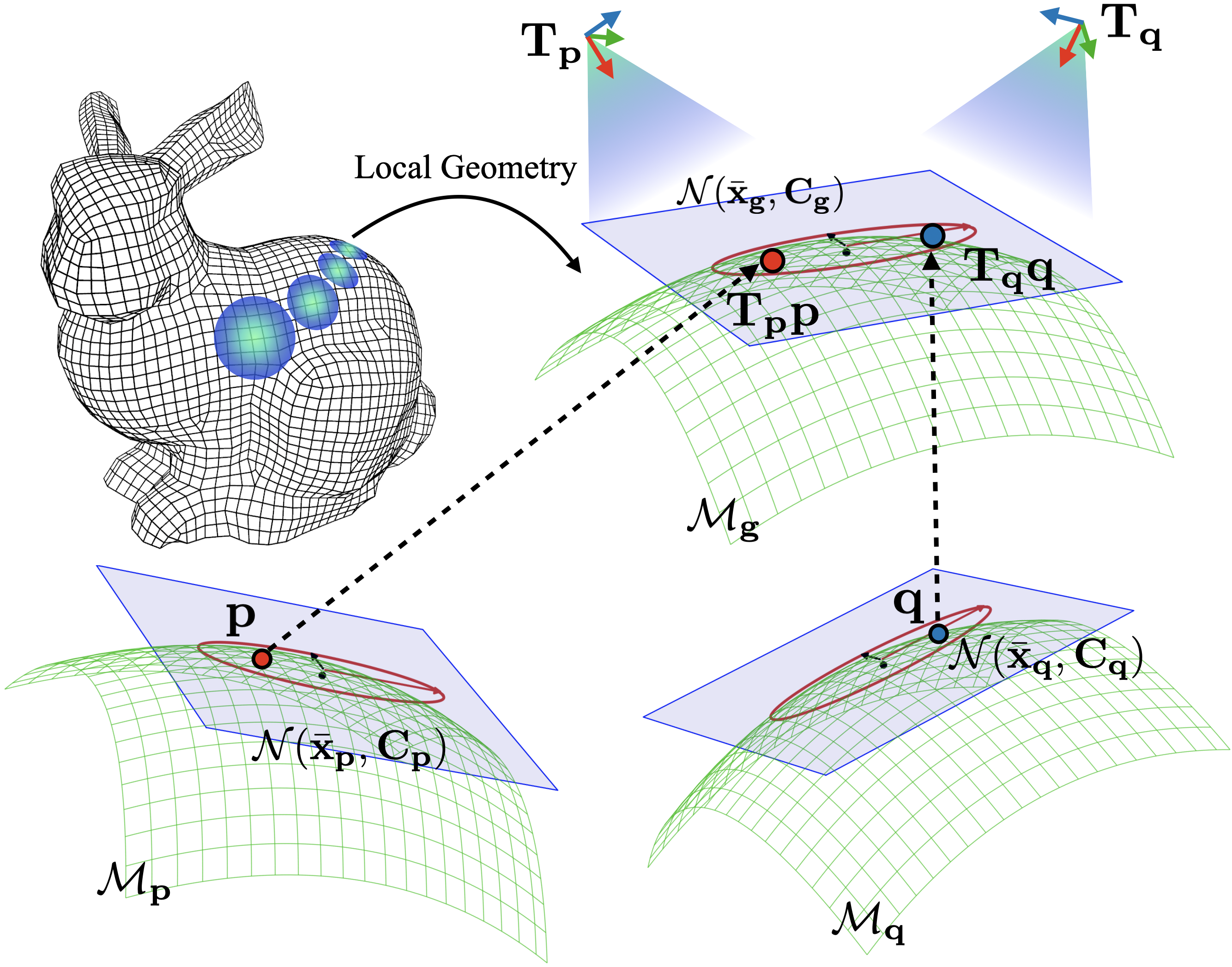}
    \caption{A visualization illustrating how the underlying geometry
can be represented as a statistical manifold. It explains that two
points observed in different coordinate frames can be regarded
as samples drawn from the same underlying distribution.}
    \label{fig:stat_mani}
\end{figure}

\begin{definition}[Statistical Manifold of World Geometry]
The statistical manifold $\mathcal{M}_{\bm{g}}$ represents the underlying world surface manifold through a probabilistic approximation in a global reference frame. The manifold $\mathcal{M}_{\bm{g}}$ is defined as a family of 3D Gaussian distributions, each parameterized by a mean $\bm{\bar{x}}_{\bm{g}_i} \in \mathbb{R}^3$ and a covariance matrix $\bm{C}_{\bm{g}_i} \in \mathbb{S}_{+}^{3}$.
Each Gaussian distribution defines the probability density that an observed point $\bm{x}_i$ is measured with respect to a local tangent plane at $\bm{\bar{x}}_{\bm{g}_i}$.
\end{definition}

When $\bm{x}_i$ is acquired from a high-precision LiDAR sensor, the resulting Gaussian distribution approximates the local planar geometry. In this formulation, the covariance matrix $\bm{C}_{\bm{g}_i}$ provides a quantitative representation of the geometry in the neighborhood of $\bm{\bar{x}}_{\bm{g}_i}$, as shown in Fig.~\ref{fig:stat_mani}.

\begin{proposition}[Statistical Manifold Representation from Different Sensor Frames]
Given the sensor pose $\bm{T}_{\bm{p}} = (\bm{R}_{\bm{p}}, \bm{t}_{\bm{p}}) \in SE(3)$ relative to the global reference frame, the parameters $\bm{\bar{x}}_{\bm{p}_i}$ and $\bm{C}_{\bm{p}_i}$ of the local statistical manifold $\mathcal{M}_{\bm{p}}$ are defined by transforming the global model $\mathcal{M}_{\bm{g}}$ into the sensor frame via $\bm{T}_{\bm{p}}^{-1}$:
\begin{equation*}
    \bm{\bar{x}}_{\bm{p}_i} \triangleq \bm{R}_{\bm{p}}^\top (\bm{\bar{x}}_{\bm{g}_i} - \bm{t}_{\bm{p}}), \qquad
    \bm{C}_{\bm{p}_i} \triangleq \bm{R}_{\bm{p}}^\top\bm{C}_{\bm{g}_i}\bm{R}_{\bm{p}}.
\end{equation*}
Analogously, for a second observation frame with pose $\bm{T}_{\bm{q}} = (\bm{R}_{\bm{q}}, \bm{t}_{\bm{q}}) \in SE(3)$, the parameters $\bm{\bar{x}}_{\bm{q}_i}$ and $\bm{C}_{\bm{q}_i}$ of the statistical manifold $\mathcal{M}_{\bm{q}}$ are defined via $\bm{T}_{\bm{q}}^{-1}$:
\begin{equation*}
    \bm{\bar{x}}_{\bm{q}_i} \triangleq \bm{R}_{\bm{q}}^\top (\bm{\bar{x}}_{\bm{g}_i} - \bm{t}_{\bm{q}}), \qquad
    \bm{C}_{\bm{q}_i} \triangleq \bm{R}_{\bm{q}}^\top \bm{C}_{\bm{g}_i}\bm{R}_{\bm{q}}.
\end{equation*}
\end{proposition}

\begin{theorem}[Residual Distribution]
\label{theorem1}
Let $\mathcal{P}\triangleq(\bm{p}_1,\dots,\bm{p}_N)$ and
$\mathcal{Q}\triangleq(\bm{q}_1,\dots,\bm{q}_N)$ denote two scans acquired
from frames $\bm{T}_{\bm{p}}$ and $\bm{T}_{\bm{q}}$, respectively, with known
correspondences $(\bm{p}_i,\bm{q}_i)$. Each point is independently sampled as
a noisy observation from its local Gaussian distribution, i.e.,
$\bm{p}_i \sim \mathcal{N}(\bm{\bar{x}}_{\bm{p}_i}, \bm{C}_{\bm{p}_i})$
and
$\bm{q}_i \sim \mathcal{N}(\bm{\bar{x}}_{\bm{q}_i}, \bm{C}_{\bm{q}_i})$,
with per-point covariances
$\bm{C}_{\bm{p}} \triangleq (\bm{C}_{\bm{p}_1},\dots,\bm{C}_{\bm{p}_N})$
and
$\bm{C}_{\bm{q}} \triangleq (\bm{C}_{\bm{q}_1},\dots,\bm{C}_{\bm{q}_N})$.
Since $\mathcal{P}$ and $\mathcal{Q}$ originate from manifolds related by an
isometric rigid-body transformation of the same underlying world surface, they
share identical local geometry, i.e., $\bm{C}_{\bm{g}_i}$. Under this local
geometric model, the displacement $\bm{d}_i$ for the $i$-th correspondence,
defined by the relative transformation
$\bm{T}\triangleq \bm{T}_{\bm{q}}^{-1}\bm{T}_{\bm{p}}$, satisfies
\begin{equation}
\label{eq:residual}
\bm{d}_i \triangleq \bm{q}_i - \bm{T}\bm{p}_i,
\qquad
\bm{d}_i \sim \mathcal{N}\!\left(\bm{0},\,2\bm{C}_{\bm{q}_i}\right).
\end{equation}
\end{theorem}

\noindent\textit{Proof.}
Let $\bm{T}=(\bm{R},\bm{t})\triangleq \bm{T}_{\bm{q}}^{-1}\bm{T}_{\bm{p}}$, where
$\bm{R}=\bm{R}_{\bm{q}}^\top\bm{R}_{\bm{p}}$ and
$\bm{t}=\bm{R}_{\bm{q}}^\top(\bm{t}_{\bm{p}}-\bm{t}_{\bm{q}})$.
Using the Gaussian perturbation form, we write
\begin{equation*}
    \bm{p}_i
    =
    \bm{R}_{\bm{p}}^\top
    (\bm{\bar{x}}_{\bm{g}_i}-\bm{t}_{\bm{p}})
    + \bm{\epsilon}_{\bm{p}_i},
    \qquad
    \bm{q}_i
    =
    \bm{R}_{\bm{q}}^\top
    (\bm{\bar{x}}_{\bm{g}_i}-\bm{t}_{\bm{q}})
    + \bm{\epsilon}_{\bm{q}_i},
\end{equation*}
where
\begin{equation*}
    \bm{\epsilon}_{\bm{p}_i}
    \sim
    \mathcal{N}(\bm{0},\bm{C}_{\bm{p}_i}),
    \qquad
    \bm{\epsilon}_{\bm{q}_i}
    \sim
    \mathcal{N}(\bm{0},\bm{C}_{\bm{q}_i}),
\end{equation*}
and the two perturbations are independent. Then,
\begin{align*}
    \bm{d}_i
    &\triangleq
    \bm{q}_i - \bm{T}\bm{p}_i \nonumber \\
    &=
    \bm{R}_{\bm{q}}^\top
    (\bm{\bar{x}}_{\bm{g}_i}-\bm{t}_{\bm{q}})
    -
    \left(
    \bm{R}\bm{R}_{\bm{p}}^\top
    (\bm{\bar{x}}_{\bm{g}_i}-\bm{t}_{\bm{p}})
    + \bm{t}
    \right)
    +
    \bm{\epsilon}_{\bm{q}_i}
    -
    \bm{R}\bm{\epsilon}_{\bm{p}_i}.
\end{align*}
Since $\bm{R}=\bm{R}_{\bm{q}}^\top\bm{R}_{\bm{p}}$ and
$\bm{t}=\bm{R}_{\bm{q}}^\top(\bm{t}_{\bm{p}}-\bm{t}_{\bm{q}})$, the deterministic terms cancel, yielding
\begin{equation*}
    \bm{d}_i
    =
    \bm{\epsilon}_{\bm{q}_i}
    -
    \bm{R}\bm{\epsilon}_{\bm{p}_i}.
\end{equation*}
Therefore,
\begin{equation*}
    \mathbb{E}[\bm{d}_i]=\bm{0},
\end{equation*}
and, by independence (i.e., sampling independently),
\begin{align*}
    \mathrm{Cov}[\bm{d}_i]
    &=
    \bm{C}_{\bm{q}_i}
    +
    \bm{R}\bm{C}_{\bm{p}_i}\bm{R}^\top \nonumber \\
    &=
    \bm{R}_{\bm{q}}^\top\bm{C}_{\bm{g}_i}\bm{R}_{\bm{q}}
    +
    \bm{R}_{\bm{q}}^\top\bm{R}_{\bm{p}}
    \left(
    \bm{R}_{\bm{p}}^\top
    \bm{C}_{\bm{g}_i}
    \bm{R}_{\bm{p}}
    \right)
    \bm{R}_{\bm{p}}^\top\bm{R}_{\bm{q}} \nonumber \\
    &=
    2\bm{R}_{\bm{q}}^\top
    \bm{C}_{\bm{g}_i}
    \bm{R}_{\bm{q}}
    =
    2\bm{C}_{\bm{q}_i}.
\end{align*}
Hence,
\begin{equation*}
    \bm{d}_i
    \sim
    \mathcal{N}\!\left(\bm{0},\,2\bm{C}_{\bm{q}_i}\right),
\end{equation*}
which proves Eq.~\eqref{eq:residual}. \hfill$\square$

\begin{corollary}[Distribution of a Point]
\label{coro:reparam}
Given the relative transformation $\bm{T}$ and source point $\bm{p}_i$, the target point $\bm{q}_i$
can be represented as
\begin{equation*}
    \bm{q}_i = \bm{T}\bm{p}_i + \bm{d}_i.
\end{equation*}
From Theorem~\ref{theorem1}, the residual follows
\begin{equation*}
    \bm{d}_i \sim \mathcal{N}(\bm{0}, 2\bm{C}_{\bm{q}_i}).
\end{equation*}
Therefore, conditioned on $\bm{p}_i$ and $\bm{T}$, the target point follows
\begin{equation*}
    \bm{q}_i \mid \bm{p}_i,\bm{T},\bm{C}_{\bm{q}_i}
    \sim
    \mathcal{N}(\bm{T}\bm{p}_i, 2\bm{C}_{\bm{q}_i}).
\end{equation*}
\end{corollary}

\begin{theorem}[Distribution of Point Cloud]
\label{thm:correspondence}
Assuming the different correspondences are conditionally independent, the conditional likelihood of the target set $\mathcal{Q}$ given the fixed source set $\mathcal{P}$ factorizes as
\begin{equation}
\begin{aligned}
p(\mathcal{Q} \mid &\mathcal{P},\bm{T},\bm{C}_{\bm{q}}) 
=
\prod_{i=1}^{N}p(\bm{q}_i \mid \bm{p}_i,\bm{T},\bm{C}_{\bm{q}_i}) \\
=
\prod_{i=1}^{N}&
\frac{1}{\sqrt{(2\pi)^3\,|2\bm{C}_{\bm{q}_i}|}}
\exp\!\left(
-\frac{1}{2}\,\bm{d}_i^\top(2\bm{C}_{\bm{q}_i})^{-1}\bm{d}_i
\right),
\end{aligned}
\label{eq:correspondence}
\end{equation}
where the last equality follows from Corollary~\ref{coro:reparam}.
\end{theorem}

%%%%%%%%%%%%%%%%%%%%%%%%%%%%%%%%%%%%%%%%%%%%%%%%%%%%%%%%%%%%%%%%%%%%%%%%%%%%%%%%%%%%%%%%%%%%%%%%%%%%%%%%%%%%%%%%%%%%%%%%%%%%%%%%%%%%%%%%%%%%%%%%%%%%%%%%%%%%%%%%%%%%%%%%%%%%%%%%%%%%%%%%%%%%%%%%%%%%%%%%%%%%%%%%%%%%%%%%%%%%%%%%%%%%%%%%%%%%%%%%%%%%%%%%%%%%%%%%%%%%%%%%%%%%%%%%%%%%%%%%%%%%%%%%%%%%%%%%%%%%%%%%%%%%%%%%%%%%%%%%%%%%%%%%%%%%%%%%%%%%%%%%%%%%%%%%%%%%%%%%%%%%%%%%%%%%%%%%%%%%%%%%%%%%%%%%%%%%%%%%%%%%%%%%%%%%%%%%%%%%%%%%%%%%%%%%%%%%%%%%%%%%%%%%%%%%%%%%%%%%%%%%%%%%%%%%%%%%%%%%%%%%%%%%%%%%%%%%%%%%%%%%%%%%%%%%

\section{Maximum Likelihood Estimation of Geometry}

In this section, we introduce two MLE methods for the local geometry (i.e., $\bm{C}_{\bm{q}_i}$).

\subsection{Correspondence Likelihood}
\label{sec:corr_likelihood}

The first estimator follows directly from the likelihood in Theorem~\ref{thm:correspondence}.

\begin{proposition}[Maximize Correspondence Likelihood]
\label{prop:corr_mle}
The maximum likelihood estimator of $\bm{C}_{\bm{q}}$ from Eq.~\eqref{eq:correspondence} is
\begin{equation}
\bm{\hat{C}}_{\bm{q}}
=
\arg\max_{\bm{C}_{\bm{q}}}\;
p(\mathcal{Q} \mid \mathcal{P},\bm{T}=\bm{T}^{\star},\bm{C}_{\bm{q}}).
\label{eq:corr_mle}
\end{equation}
We refer to this formulation as the maximization of correspondence likelihood, as it is defined over given pairs of corresponding points and characterizes the residual between them.
\end{proposition}

Taking the negative log-likelihood of Eq.~\eqref{eq:corr_mle} and discarding constants reduces the MLE to
\begin{equation}
\bm{\hat{C}}_{\bm{q}}
=
\arg\min_{\bm{C}_{\bm{q}}}\;
\tfrac{1}{2}\sum_{i=1}^{N}
\left[
\log\lvert 2\bm{C}_{\bm{q}_{i}}\rvert
+
\bm{d}_{i}^{\star\top}
(2\bm{C}_{\bm{q}_{i}})^{-1}
\bm{d}^{\star}_{i}
\right],
\label{eq:corr_nll}
\end{equation}
where $\bm{d}^{\star}_{i} = \bm{q}_{i} - \bm{T}^{\star}\bm{p}_{i}$. In practice, the ground-truth pose $\bm{T}^{\star}$ is rarely available; instead, it can be obtained from LiDAR SLAM.

\subsection{Pose Likelihood}
\label{sec:pose_likelihood}

The second estimator takes a different view. We ask: \emph{given a candidate $\bm{C}_{\bm{q}}$, does the pose implied by the correspondence likelihood agree with the observed pose $\bm{\tilde{T}}$ (e.g., from LiDAR SLAM)?} To formalize this, we first define the pose implied by the correspondence likelihood for a given covariance set:
\begin{equation}
\bm{T}(\bm{C}_{\bm{q}})
\triangleq
\arg\max_{\bm{T}}\;
p(\mathcal{Q} \mid \mathcal{P},\bm{T},\bm{C}_{\bm{q}}),
\label{eq:tc_def}
\end{equation}
which is equivalent to the weighted least-squares problem
\begin{equation*}
\bm{T}(\bm{C}_{\bm{q}})
=
\arg\min_{\bm{T}}\;
\sum_{i=1}^{N}
\bm{d}_{i}(\bm{T})^{\top}(2\bm{C}_{\bm{q}_{i}})^{-1}\bm{d}_{i}(\bm{T}).
\end{equation*}
We now characterize the sampling distribution of $\bm{T}(\bm{C}_{\bm{q}})$.

\begin{proposition}[Distribution of the Correspondence-Implied Pose]
\label{prop:tc_dist}
The pose estimator in Eq.~\eqref{eq:tc_def} is the weighted least-squares MLE of $\bm{T}^{\star}$. By the standard asymptotic theory of MLE,
\begin{equation*}
\bm{T}(\bm{C}_{\bm{q}})
=
\bm{T}^{\star}\mathrm{Exp}(\bm{\eta}),
\quad
\bm{\eta} \sim \mathcal{N}\!\left(\bm{0},\,\bm{H}^{-1}(\bm{C}_{\bm{q}})\right),
\label{eq:tc_dist}
\end{equation*}
where $\mathrm{Exp}(\cdot)$ denotes the exponential map of $SE(3)$, which lifts a $\mathbb{R}^{6}$ Lie algebra perturbation to a relative transformation, and $\bm{H}(\bm{C}_{\bm{q}})$ is the Hessian of the weighted least-squares objective at $\bm{T}(\bm{C}_{\bm{q}})$.
\end{proposition}
We additionally adopt a Gaussian pose observation model for $\bm{\tilde{T}}$, which is typically obtained from a measurement (e.g., pose from LiDAR SLAM):
\begin{equation}
\bm{\tilde{T}} = \bm{T}^{\star}\mathrm{Exp}(\bm{\xi}),
\qquad
\bm{\xi} \sim \mathcal{N}(\bm{0},\bm{\Gamma}),
\label{eq:pose_obs}
\end{equation}
where $\bm{\xi} \in \mathbb{R}^{6}$ is the pose noise in the Lie algebra and $\bm{\Gamma}$ is its covariance. Combining Eq.~\eqref{eq:pose_obs} and Proposition~\ref{prop:tc_dist} yields the distribution of the discrepancy between the correspondence-implied pose and the observed pose.
\begin{proposition}[Distribution of the Pose Discrepancy]
\label{prop:pose_discrepancy}
Under the independence of the pose measurement noise $\bm{\xi}$ and the estimation noise $\bm{\eta}$, the discrepancy between $\bm{T}(\bm{C}_{\bm{q}})$ and $\bm{\tilde{T}}$ in the Lie algebra satisfies
\begin{equation*}
\mathrm{Log}\!\left(\bm{T}(\bm{C}_{\bm{q}})^{-1}\bm{\tilde{T}}\right)
\approx
\bm{\xi} - \bm{\eta}
\sim
\mathcal{N}\!\left(\bm{0},\,\bm{W}(\bm{C}_{\bm{q}})\right),
\label{eq:disc_dist}
\end{equation*}
where the approximation follows from the first-order Baker--Campbell--Hausdorff (BCH) expansion, and
\begin{equation*}
\bm{W}(\bm{C}_{\bm{q}})
\triangleq
\bm{\Gamma} + \bm{H}^{-1}(\bm{C}_{\bm{q}}).
\label{eq:W_def}
\end{equation*}
\end{proposition}
Here, $\mathrm{Log}(\cdot)$ denotes the logarithm map of $SE(3)$, which maps a relative transformation to its $\mathbb{R}^{6}$ Lie algebra representation.

\begin{corollary}[Distribution of the Observed Pose]
\label{coro:tilde_T_dist}
Given the correspondence-implied pose $\bm{T}(\bm{C}_{\bm{q}})$ and the discrepancy covariance $\bm{W}(\bm{C}_{\bm{q}})$, the observed pose $\bm{\tilde{T}}$ can be represented as
\begin{equation*}
\bm{\tilde{T}} = \bm{T}(\bm{C}_{\bm{q}})\,\mathrm{Exp}(\bm{\mu}),
\end{equation*}
where, from Proposition~\ref{prop:pose_discrepancy}, the discrepancy follows
\begin{equation*}
\bm{\mu} \sim \mathcal{N}(\bm{0},\bm{W}(\bm{C}_{\bm{q}})).
\end{equation*}
Therefore, conditioned on $\bm{T}(\bm{C}_{\bm{q}})$ and $\bm{W}(\bm{C}_{\bm{q}})$, the observed pose follows
\begin{equation*}
\bm{\tilde{T}} \mid \bm{T}(\bm{C}_{\bm{q}}),\bm{W}(\bm{C}_{\bm{q}})
\sim
\mathcal{N}\!\left(\bm{T}(\bm{C}_{\bm{q}}),\,\bm{W}(\bm{C}_{\bm{q}})\right),
\end{equation*}
in the sense that its Lie algebra perturbation about $\bm{T}(\bm{C}_{\bm{q}})$ is Gaussian.
\end{corollary}

\begin{proposition}[Maximize Pose Likelihood]
\label{prop:pose_mle}
The maximum likelihood estimator of $\bm{C}_{\bm{q}}$ from the distribution in Corollary~\ref{coro:tilde_T_dist} is
\begin{equation}
\bm{\hat{C}}_{\bm{q}}
=
\arg\max_{\bm{C}_{\bm{q}}}\;
p\!\left(\bm{\tilde{T}} \mid \bm{T}(\bm{C}_{\bm{q}}),\bm{W}(\bm{C}_{\bm{q}})\right).
\label{eq:pose_mle}
\end{equation}
We refer to this as the maximization of pose likelihood, as it is defined over the discrepancy between the correspondence-implied pose $\bm{T}(\bm{C}_{\bm{q}})$ and the observed pose $\bm{\tilde{T}}$.
\end{proposition}
Taking the negative log-likelihood of Eq.~\eqref{eq:pose_mle} and discarding constants reduces the MLE to
\begin{equation}
\bm{\hat{C}}_{\bm{q}}
=
\arg\min_{\bm{C}_{\bm{q}}}\;
\tfrac{1}{2}\,\bm{\mu}(\bm{C}_{\bm{q}})^{\top}\bm{W}(\bm{C}_{\bm{q}})^{-1}\bm{\mu}(\bm{C}_{\bm{q}})
\,+\,
\tfrac{1}{2}\log\lvert\bm{W}(\bm{C}_{\bm{q}})\rvert,
\label{eq:pose_nll}
\end{equation}
where $\bm{\mu}(\bm{C}_{\bm{q}}) \triangleq \mathrm{Log}\!\left(\bm{T}(\bm{C}_{\bm{q}})^{-1}\bm{\tilde{T}}\right)$.

%%%%%%%%%%%%%%%%%%%%%%%%%%%%%%%%%%%%%%%%%%%%%%%%%%%%%%%%%%%%%%%%%%%%%%%%%%%%%%%%%%%%%%%%%%%%%%%%%%%%%%%%%%%%%%%%%%%%%%%%%%%%%%%%%%%%%%%%%%%%%%%%%%%%%%%%%%%%%%%%%%%%%%%%%%%%%%%%%%%%%%%%%%%%%%%%%%%%%%%%%%%%%%%%%%%%%%%%%%%%%%%%%%%%%%%%%%%%%%%%%%%%%%%%%%%%%%%%%%%%%%%%%%%%%%%%%%%%%%%%%%%%%%%%%%%%%%%%%%%%%%%%%%%%%%%%%%%%%%%%%%%%%%%%%%%%%%%%%%%%%%%%%%%%%%%%%%%%%%%%%%%%%%%%%%%%%%%%%%%%%%%%%%%%%%%%%%%%%%%%%%%%%%%%%%%%%%%%%%%%%%%%%%%%%%%%%%%%%%%%%%%%%%%%%%%%%%%%%%%%%%%%%%%%%%%%%%%%%%%%%%%%%%%%%%%%%%%%%%%%%%%%%%%%%%%%
\section{Geometry Reasoning for LiDAR SLAM}

Building on the above likelihood-based formulation, we now instantiate geometry reasoning within a LiDAR SLAM pipeline. The proposed framework consists of three components: a \textit{neural module} that predicts local geometry from raw point clouds, a \textit{reasoning} module that evaluates geometric likelihoods, and a shared \textit{memory} that stores the data exchanged between them. These components are coupled through three procedures operating at different scales: Algorithm~\ref{alg:training} trains the neural module from a fixed memory snapshot; Algorithm~\ref{alg:downstream} uses the trained module to refine the memory through SLAM; and Algorithm~\ref{alg:reciprocal} alternates the two until convergence. We describe each in turn.\\

\noindent\textbf{Memory.}
We define a memory module $\mathbf{M}$ as
\begin{equation*}
\mathbf{M} \triangleq
\bigl\{
\{\mathcal{S}_k\}_{k=1}^{K},\;
\hat{\mathcal{T}},\;
\mathcal{C}
\bigr\},
\label{eq:memory}
\end{equation*}
where $\mathcal{S}_k$ is the $k$-th LiDAR scan. Given a scan-pair graph
$\mathcal{E}$, the pose and correspondence memories are defined as
\begin{equation*}
\hat{\mathcal{T}}
\triangleq
\{\hat{\bm{T}}_{k,l}\}_{(k,l)\in\mathcal{E}},
\qquad
\mathcal{C}
\triangleq
\{\mathcal{C}_{k,l}\}_{(k,l)\in\mathcal{E}},
\end{equation*}
where $\hat{\bm{T}}_{k,l}\in SE(3)$ is the estimated relative pose from
$\mathcal{S}_k$ to $\mathcal{S}_l$, and
$\mathcal{C}_{k,l}=\{(\bm{p}_i,\bm{q}_i)\mid
\bm{p}_i\in\mathcal{S}_k,\bm{q}_i\in\mathcal{S}_l\}$ is the corresponding
point-pair set. The memory provides scans, poses, and correspondences to the
neural and spatial reasoning modules, and is updated by the SLAM module
$\mathcal{F}_{\mathrm{SLAM}}$ during reciprocal learning.\\

\begin{algorithm}[t]
\caption{Self-supervised Training of Geometry Reasoning}
\label{alg:training}
\begin{algorithmic}[1]
\Require Neural module $f_{\bm{\theta}}$, memory $\mathbf{M}=(\{\mathcal{S}_k\}_{k=1}^{K},\hat{\mathcal{T}},\mathcal{C})$, hyperparams $\Omega=\{\eta,I,\lambda_{\mathrm{corr}},\lambda_{\mathrm{pose}},\epsilon\}$
\Ensure Trained neural net $f_{\bm{\theta}^{\star}}$
\For{$i=1$ to $I$}
    \State $\mathcal{L}(\bm{\theta})\leftarrow 0$
    \For{each pair $(k,l)\in\mathcal{E}$}
        \State $\mathcal{P}\leftarrow\mathcal{S}_k,\;\mathcal{Q}\leftarrow\mathcal{S}_l$
        \State $\hat{\bm{T}}_{k,l}\leftarrow\hat{\mathcal{T}}_{k,l},\quad
        \mathcal{C}_{k,l}\leftarrow\mathcal{C}(k,l)$
        \State {\textcolor{blue}{/* $\mathcal{Q}\to\bm{C}_{\bm{q}}$ via Cholesky parameterization */}}
        \State $\bm{L}_{\bm{q}}\leftarrow f_{\bm{\theta}}(\mathcal{Q})$, $\bm{C}_{\bm{q}}\leftarrow
        \bm{L}_{\bm{q}}\bm{L}_{\bm{q}}^{\top}+\epsilon\bm{I}$
        \State {\textcolor{blue}{/* Correspondence likelihood, Eq.~(\ref{eq:corr_nll}) */}}
        \State $\mathcal{L}_{\mathrm{corr}}\leftarrow
        \sum_i
        \left[
        \log\lvert 2\bm{C}_{\bm{q}_i}\rvert+
        \bm{\tilde d}_i^{\top}
        (2\bm{C}_{\bm{q}_i})^{-1}
        \bm{\tilde d}_i
        \right]$
        \State {\textcolor{blue}{/* Pose likelihood, Eq.~(\ref{eq:pose_nll}) */}}
        \State $\bm{T}(\bm{C}_{\bm{q}})\leftarrow
        \texttt{SDPRLayer}(\bm{C}_{\bm{q}},\mathcal{C}_{k,l},\mathcal{P},\mathcal{Q})$
        \State $\mathcal{L}_{\mathrm{pose}}\leftarrow
        \bm{\mu}^{\top}\bm{W}^{-1}\bm{\mu}
        +
        \log\lvert\bm{W}\rvert$
        \State $\mathcal{L}(\bm{\theta})\mathrel{+}=
        \lambda_{\mathrm{corr}}\mathcal{L}_{\mathrm{corr}}
        +
        \lambda_{\mathrm{pose}}\mathcal{L}_{\mathrm{pose}}$
         \State $\bm{\theta}\leftarrow
    \bm{\theta}-\eta\,\nabla_{\bm{\theta}}\mathcal{L}(\bm{\theta})$
    \EndFor
\EndFor
\State \Return $f_{\bm{\theta}^{\star}}$
\end{algorithmic}
\end{algorithm}

\begin{algorithm}[t]
\caption{SLAM with Reasoned Geometry}
\label{alg:downstream}
\begin{algorithmic}[1]
\Require Memory $\mathbf{M}=(\{\mathcal{S}_k\},\hat{\mathcal{T}},\mathcal{C})$, trained net $f_{\bm{\theta}^{\star}}$, SLAM module $\mathcal{F}_{\mathrm{SLAM}}$, hyperparams $\Omega=\{\tau_{\mathrm{ani}},M\}$
\Ensure Updated memory $\mathbf{M}^{+}=(\{\mathcal{S}^{+}_k\},\hat{\mathcal{T}}^{+},\mathcal{C}^{+})$
\For{$k=1$ to $K$}
   \State {\textcolor{blue}{/* Construct positive-definite covariance */}} 
   \State $\bm{L}_{k}\leftarrow f_{\bm{\theta}^{\star}}(\mathcal{S}_k)$, $\bm{C}_k\leftarrow \bm{L}_{k}\bm{L}_{k}^{\top}+\epsilon\bm{I}$
    \State $\mathcal{S}^{+}_k\leftarrow\emptyset$
    \For{each point $\bm{s}_{k,i}\in\mathcal{S}_k$}
        \State Compute anisotropy $a_{k,i}$ from $\bm{C}_{k,i}$
        \State {\textcolor{blue}{/* Symbolic-grounded outlier rejection */}}
        \If{$a_{k,i}>\tau_{\mathrm{ani}}$}
            \State Sample $\{\bm{s}_{k,i}^{(m)}\}_{m=1}^{M}\sim\mathcal{N}(\bm{s}_{k,i},\bm{C}_{k,i})$
            \State $\mathcal{S}^{+}_k\leftarrow\mathcal{S}^{+}_k\cup\{\bm{s}_{k,i}^{(m)}\}_{m=1}^{M}$
        \EndIf
    \EndFor
\EndFor
\State {\textcolor{blue}{/* Memory update */}}
\State $\hat{\mathcal{T}}^{+}\leftarrow\mathcal{F}_{\mathrm{SLAM}}(\{\mathcal{S}^{+}_k\})$
\State $\mathcal{C}^+ \gets \texttt{UpdateCorrespondence}(\hat{\mathcal{T}}^+,\mathcal{C})$
\State \Return $\mathbf{M}^{+}=(\mathcal{S}_k,\hat{\mathcal{T}}^{+},\mathcal{C}^{+})$
\end{algorithmic}
\end{algorithm}

\noindent\textbf{Self-supervised training of the neural module.}
Given the current memory $\mathbf{M}$, Algorithm~\ref{alg:training} trains $f_{\bm{\theta}}$ to predict point-wise covariances that are consistent with the trajectory and correspondences stored in memory. The training signal is composed of two likelihoods. The (\emph{i}) correspondence likelihood (line~8) asks whether the predicted covariance statistically explains the residual between corresponding points, following the Gaussian distribution established in Proposition~\ref{prop:corr_mle}. The (\emph{ii}) pose likelihood (line~11) asks whether the trajectory implied by the predicted covariances agrees with the trajectory already in memory, following Proposition~\ref{prop:pose_mle}. Evaluating the pose likelihood requires the global optimum of a weighted least-squares problem, as well as gradients through that optimum. We obtain both using the SDPR layer~\citep{holmes2025sdprlayers}, which leverages the implicit function theorem (IFT) to differentiate through the optimization problem (line~10).

\begin{remark}[Self-supervision from Symbolic Consistency]
No ground-truth covariance label is required: the supervision arises entirely from the consistency between the predicted geometry, correspondences, and trajectory.
\end{remark}

\begin{algorithm}[t]
\caption{Reciprocal Learning}
\label{alg:reciprocal}
\begin{algorithmic}[1]
\Require Raw scans $\{\mathcal{S}_k\}_{k=1}^{K}$, initial net $f_{\bm{\theta}^{(0)}}$, SLAM module $\mathcal{F}_{\mathrm{SLAM}}$, hyperparams $\Omega$, threshold $\delta$
\Ensure Refined memory $\mathbf{M}^{\star}$, trained net $f_{\bm{\theta}^{\star}}$
\State {\textcolor{blue}{/* Initialize memory via SLAM */}}
\State $(\hat{\mathcal{T}}^{(0)},\mathcal{C}^{(0)})\leftarrow\mathcal{F}_{\mathrm{SLAM}}(\{\mathcal{S}_k\})$
\State $\mathbf{M}^{(0)}\leftarrow(\{\mathcal{S}_k\},\hat{\mathcal{T}}^{(0)},\mathcal{C}^{(0)})$
\State $r\leftarrow 0$
\While{not converged}
    \State {\textcolor{blue}{/* Update neural module from memory */}}
    \State $f_{\bm{\theta}^{(r+1)}}\leftarrow\textbf{Alg.~\ref{alg:training}}(f_{\bm{\theta}^{(r)}},\mathbf{M}^{(r)},\Omega)$
    \State {\textcolor{blue}{/* Update memory via reasoning */}}
    \State $\mathbf{M}^{(r+1)}\leftarrow\textbf{Alg.~\ref{alg:downstream}}(\mathbf{M}^{(r)},f_{\bm{\theta}^{(r+1)}},\mathcal{F}_{\mathrm{SLAM}},\Omega)$
    \State \textbf{if} $\|\hat{\mathcal{T}}^{(r+1)}-\hat{\mathcal{T}}^{(r)}\|<\delta$ \textbf{then} converged
    \State $r\leftarrow r+1$
\EndWhile
\State \Return $\mathbf{M}^{\star}=\mathbf{M}^{(r)},\;f_{\bm{\theta}^{\star}}=f_{\bm{\theta}^{(r)}}$
\end{algorithmic}
\end{algorithm}

\begin{figure}[t]
\centering
\begin{subfigure}[t]{0.95\columnwidth}
    \centering
    \includegraphics[
        width=\linewidth,
        trim=15 15 15 15,
        clip
    ]{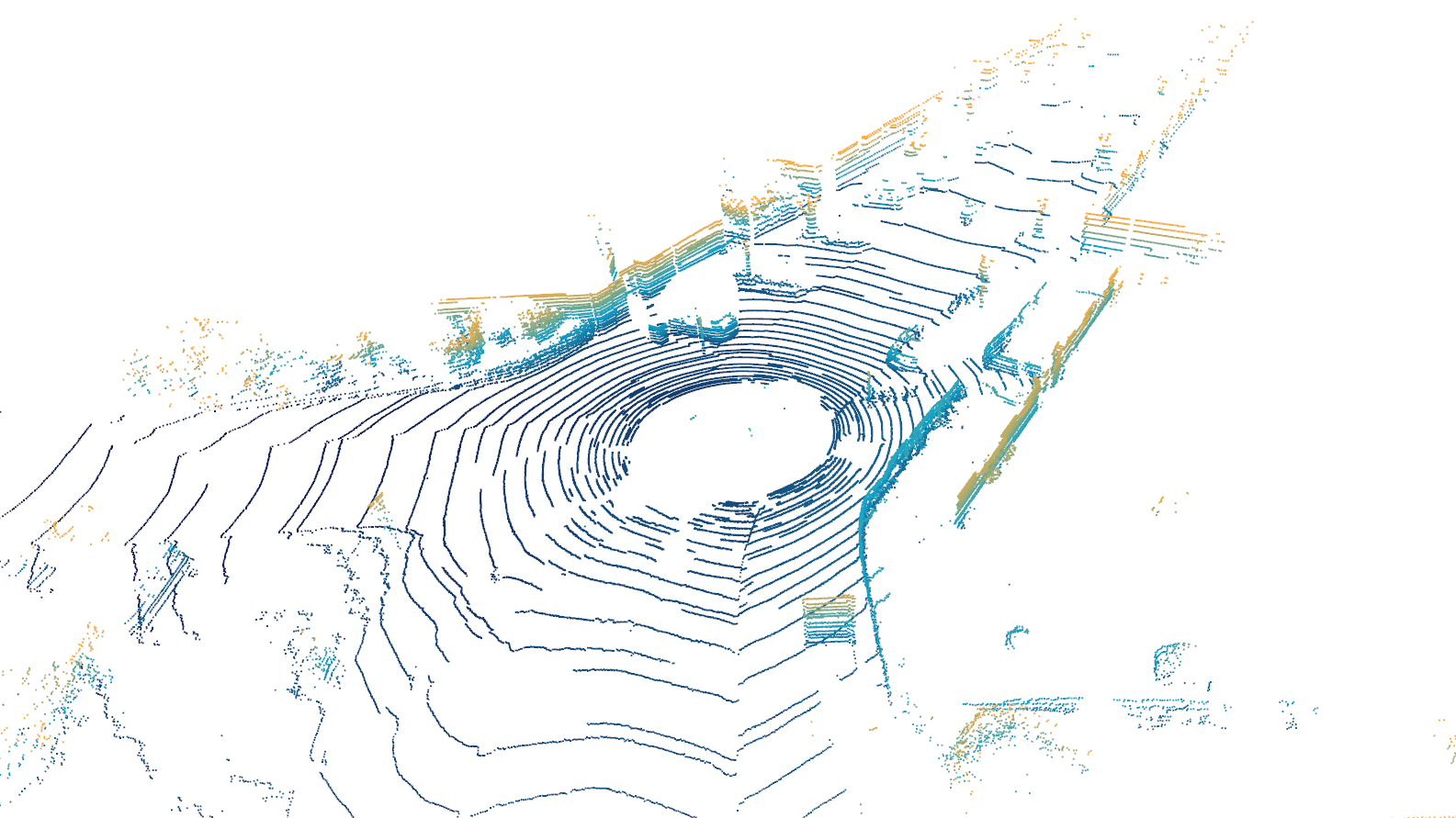}
    \caption{Raw scan}
    \label{fig:raw}
\end{subfigure}

\vspace{1mm}

\begin{subfigure}[t]{0.95\columnwidth}
    \centering
    \includegraphics[
        width=\linewidth,
        trim=15 15 15 15,
        clip
    ]{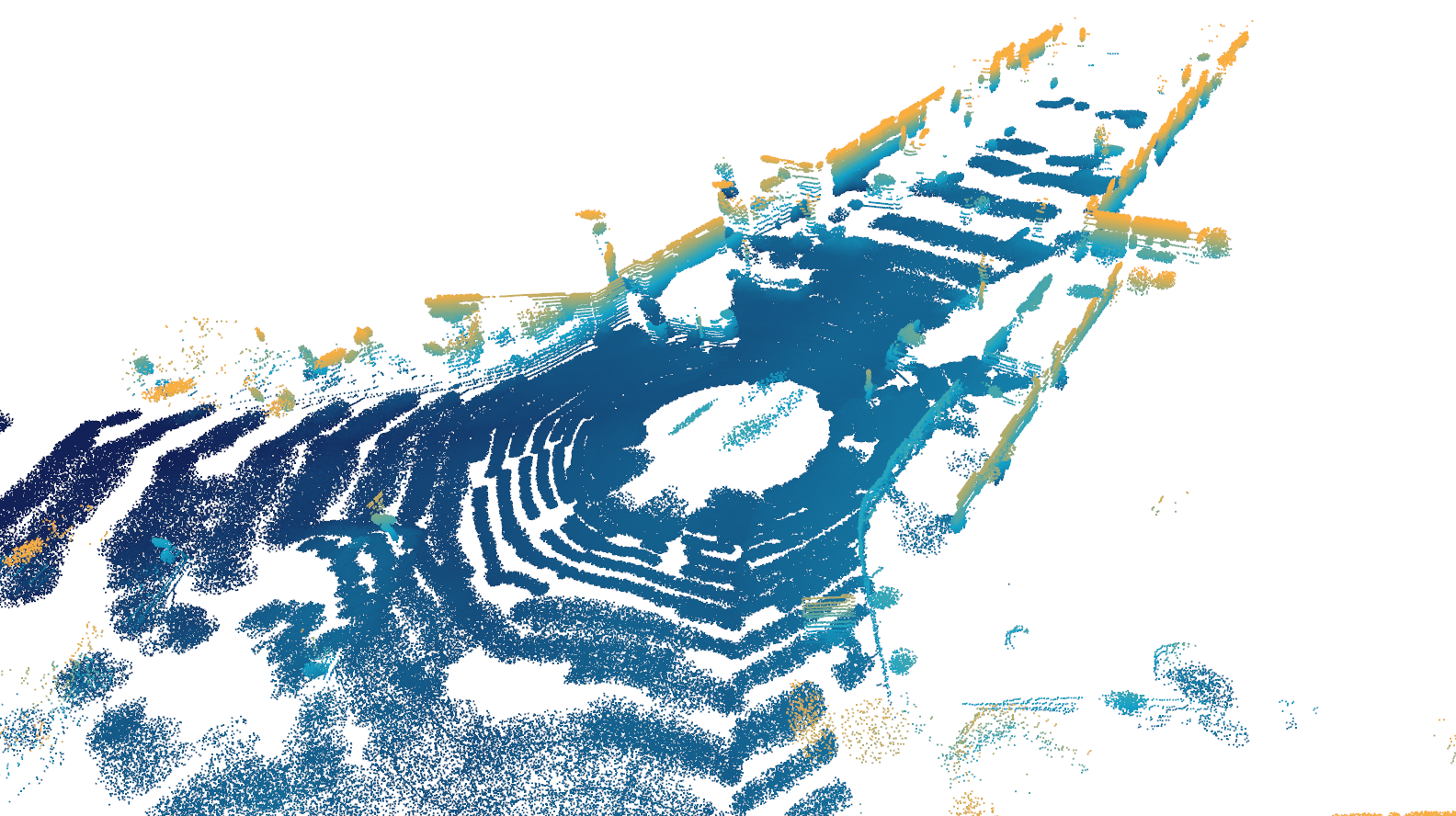}
    \caption{Densified scan}
    \label{fig:dense}
\end{subfigure}

\caption{Raw and densified point clouds.
(a) Raw 32-channel scan.
(b) Densified scan from learned covariance.}
\label{fig:densify}
\end{figure}

\begin{figure}[t]
    \centering
    \includegraphics[width=0.98\linewidth, trim= 20 30 20 0, clip]{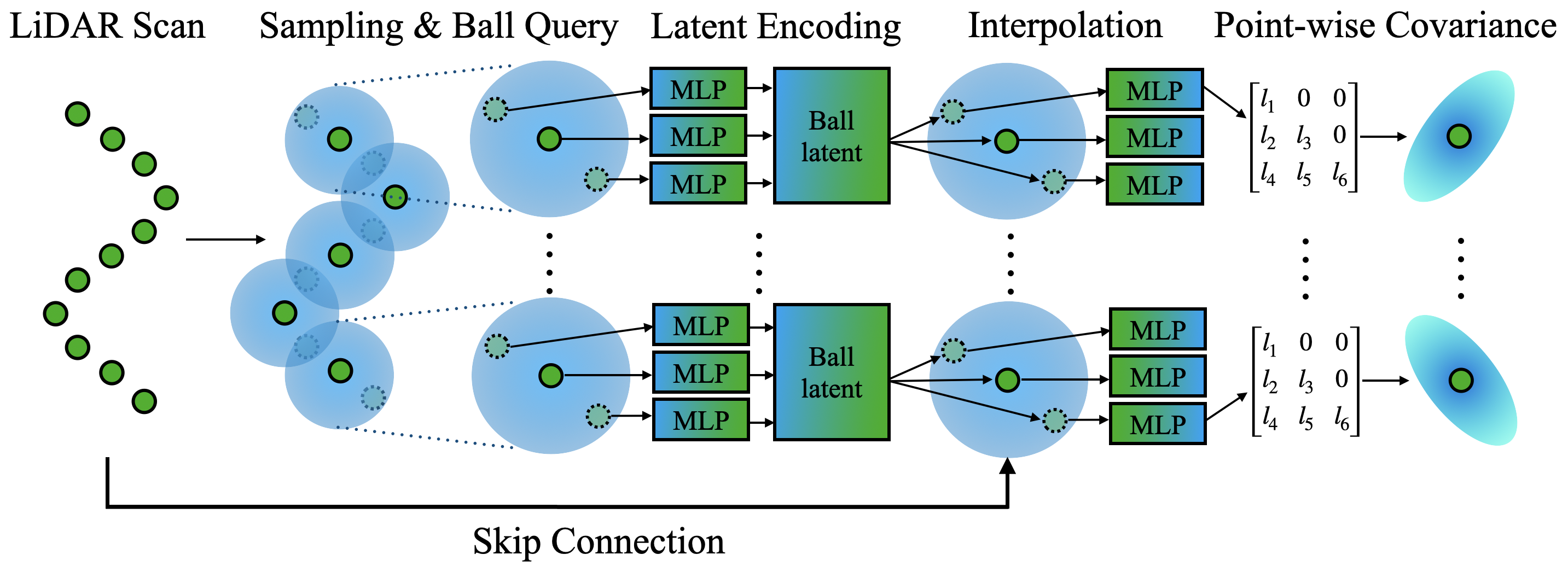}
    \caption{Network architecture of covariance estimator.}
    \label{fig:network_architecture}
\end{figure}

\noindent\textbf{SLAM with reasoned geometry.}
Once $f_{\bm{\theta}^{\star}}$ is trained, Algorithm~\ref{alg:downstream} applies it to every scan and refines the memory accordingly. For each scan, the network predicts per-point covariances (line~3), and each point is filtered by its anisotropy (line~8). As established in Section~\ref{sec:3}, points sampled from local surfaces under high-precision LiDAR induce anisotropic covariances; near-isotropic predictions therefore indicate that the Gaussian model is not faithful to the underlying geometry. Retained points are resampled from the predicted distributions (line~9) to form geometry-enhanced scans, qualitatively shown in Fig.~\ref{fig:densify}. These scans are then passed to $\mathcal{F}_{\mathrm{SLAM}}$ to update the trajectory and correspondences (lines~15--16). Note that the SLAM module is treated as a black box, so any LiDAR SLAM pipeline can be plugged in.\\

\noindent\textbf{Reciprocal learning.}
Algorithms~\ref{alg:training} and~\ref{alg:downstream} depend on each other: the neural module is trained from the memory, and the memory is refined using the neural module. Algorithm~\ref{alg:reciprocal} resolves this circularity by alternating between the two, starting from an initial memory obtained by running SLAM on the raw scans (lines~2--3). Each iteration updates the neural module (line~7) and then the memory (line~9), and the loop terminates when the trajectory change falls below $\delta$ (line~10).

\begin{figure}[t]
\centering
\begin{subfigure}[t]{0.48\columnwidth}
    \centering
    \includegraphics[width=\linewidth]{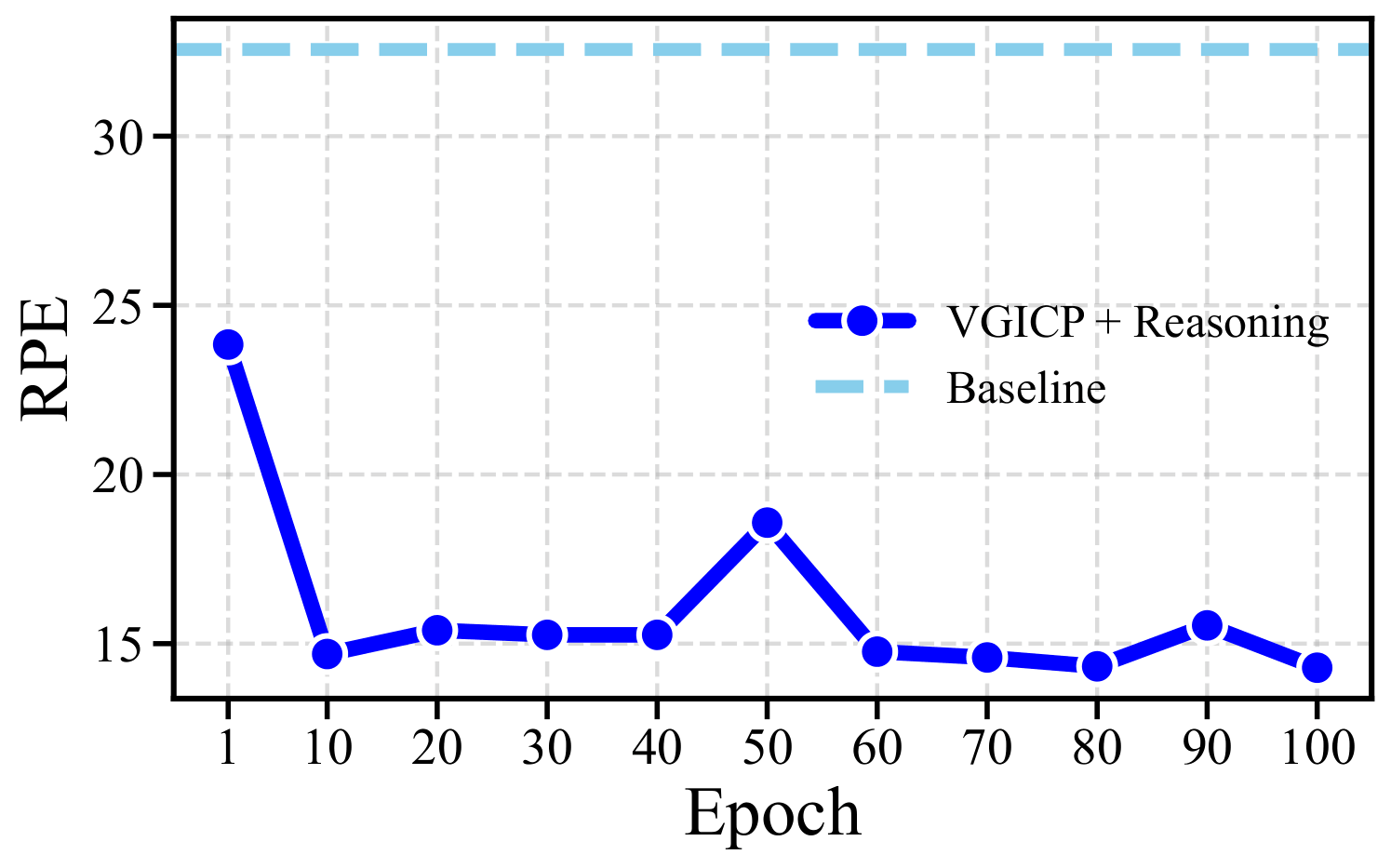}
    \caption{RPE over epochs}
    \label{fig:rpe_vs_epoch}
\end{subfigure}
\hfill
\begin{subfigure}[t]{0.48\columnwidth}
    \centering
    \includegraphics[width=\linewidth]{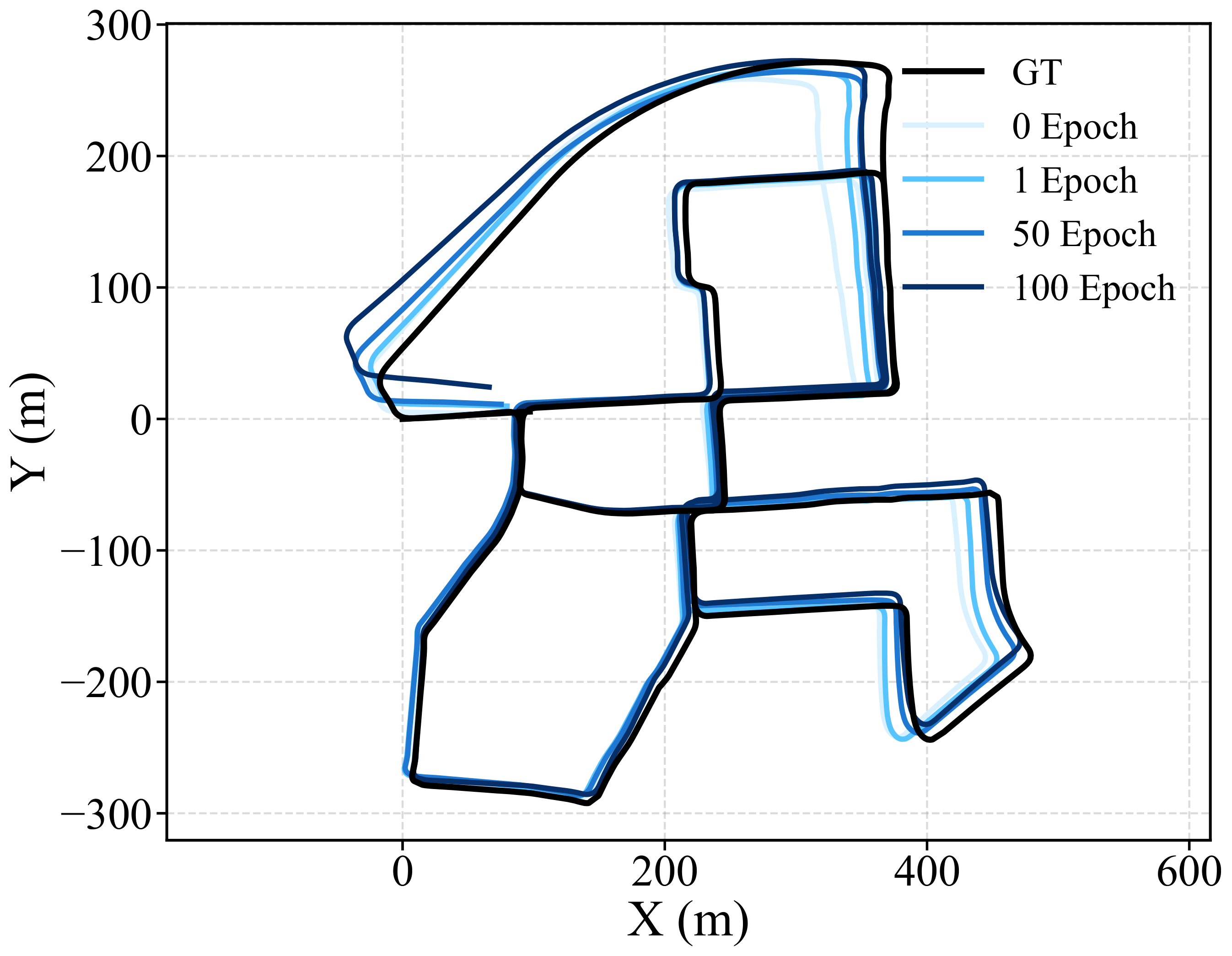}
    \caption{Trajectory comparison}
    \label{fig:epoch_trajectory}
\end{subfigure}

\caption{Odometry performance over different training epochs.
(a) Translational RPE over 300-frame intervals.
(b) Estimated trajectories at the corresponding epochs.}
\label{fig:epoch_analysis}
\end{figure}

\begin{table}[t]
\centering
\caption{Odometry performance under different LiDAR resolutions. We report translational RPE over 300-frame intervals, with relative error reductions from Raw in parentheses.}
\label{tab:odom_channel}
\footnotesize
\setlength{\tabcolsep}{3.5pt}
\renewcommand{\arraystretch}{1.08}
\setlength{\fboxsep}{1pt}
\begin{tabular}{lccc}
\toprule
\textbf{Input} & \textbf{16 ch.} & \textbf{32 ch.} & \textbf{64 ch.} \\
\midrule
Raw
& 15.55
& 7.88
& 3.68 \\
1-step
& \colorbox{blue!12}{10.02 ($\downarrow$35.6\%)}
& \colorbox{blue!15}{4.80 ($\downarrow$39.1\%)}
& \colorbox{blue!6}{3.50 ($\downarrow$4.9\%)} \\
2-step
& \colorbox{blue!25}{\textbf{7.85} ($\downarrow$49.5\%)}
& \colorbox{blue!25}{\textbf{4.11} ($\downarrow$47.8\%)}
& \colorbox{blue!8}{\textbf{3.46} ($\downarrow$6.0\%)} \\
\bottomrule
\end{tabular}
\end{table}

\section{Experiments}

\subsection{Experimental Setup}

We validate the proposed framework on KITTI odometry Sequence~00~\citep{geiger2012we} under 64-, 32-, and 16-channel LiDAR settings. The 32- and 16-channel scans are generated by downsampling the original 64 LiDAR channels with strides of 2 and 4, respectively.  We use a lightweight VGICP-based odometry pipeline~\citep{koide2021voxelized} as $\mathcal{F}_{\mathrm{SLAM}}$ and a PointNet++~\citep{qi2017pointnet++}-inspired architecture as the neural module $f_{\bm{\theta}}$ (Fig.~\ref{fig:network_architecture}). Starting from the raw memory information obtained from VGICP, we run two reciprocal rounds as a minimal experimental setting to examine the reciprocal effect between geometry reasoning and SLAM. Raw scans are replaced with densified scans generated from the reasoned underlying geometry, denoted as \emph{+~Reasoning} in the following experiments, as shown in Algorithm~\ref{alg:downstream} lines~3--9, and Fig.~\ref{fig:densify}.

\subsection{Evaluation Protocol}
We assess both odometry and global registration. For odometry, we compare the initial VGICP trajectory (\emph{Raw}) against the trajectories after one (\emph{1-step}) and two (\emph{2-step}) reciprocal updates, using the translational relative pose error (RPE) over 300-frame intervals. For global registration, we randomly sample 100 scan pairs whose ground-truth translation distance is within $10$m and align them with TEASER~\citep{yang2020teaser}, counting a pair as successful when the rotation error is below $10^{\circ}$ and the translation error is below $2$m.

\subsection{Results and Analysis}

\noindent\textbf{Odometry.}
Table~\ref{tab:odom_channel} shows that geometry reasoning reduces translational drift at all resolutions, with larger gains for sparser inputs. At 16 and 32 channels, the 2-step model cuts drift by roughly half ($49.5\%$ and $47.8\%$), while the gain at 64 channels is marginal ($6.0\%$) because the dense scans already provide reliable geometry. The second reciprocal round further improves odometry in all settings, indicating progressive trajectory refinement (Fig.~\ref{fig:epoch_analysis}). \\

\begin{figure}[t]
\centering
\begin{subfigure}[t]{0.48\columnwidth}
    \centering
    \includegraphics[
        width=\linewidth,
        trim={15 15 15 15},
        clip
    ]{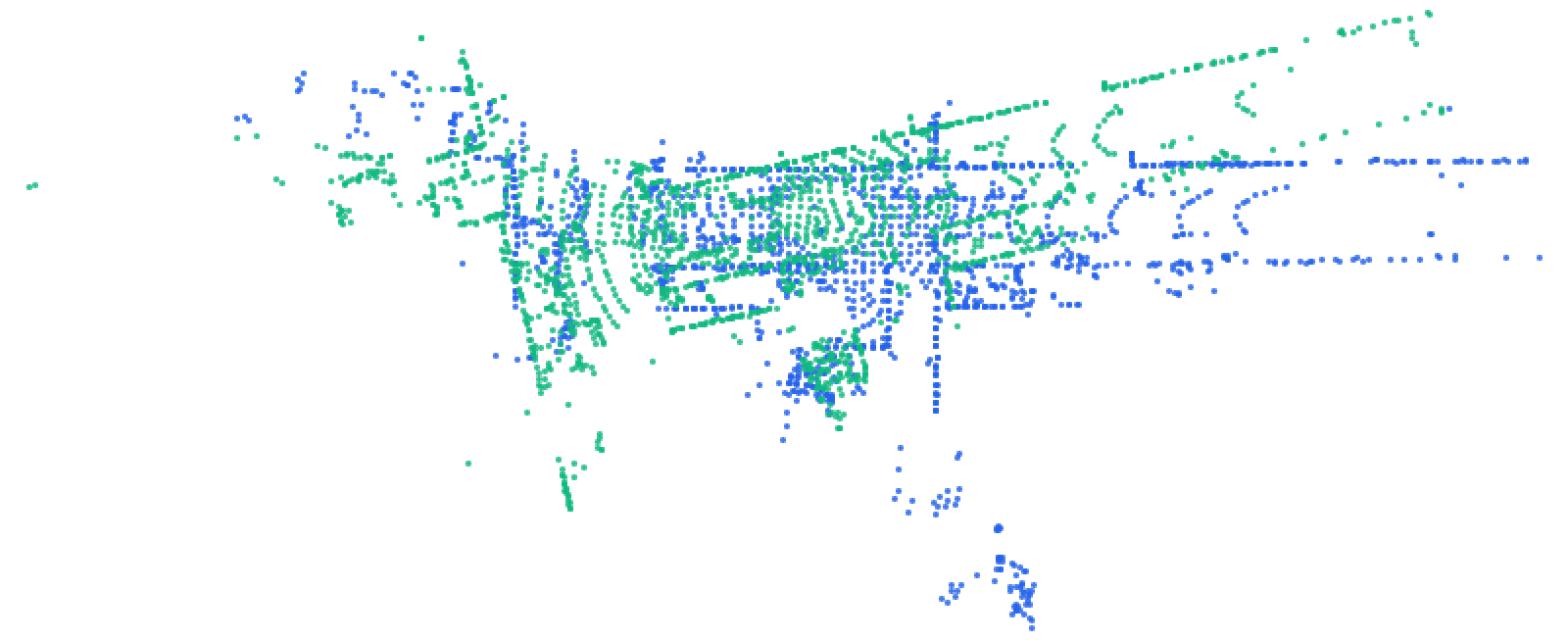}
    \caption{TEASER}
    \label{fig:reg_fail}
\end{subfigure}
\hfill
\begin{subfigure}[t]{0.48\columnwidth}
    \centering
    \includegraphics[
        width=\linewidth,
        trim={15 15 15 15},
        clip
    ]{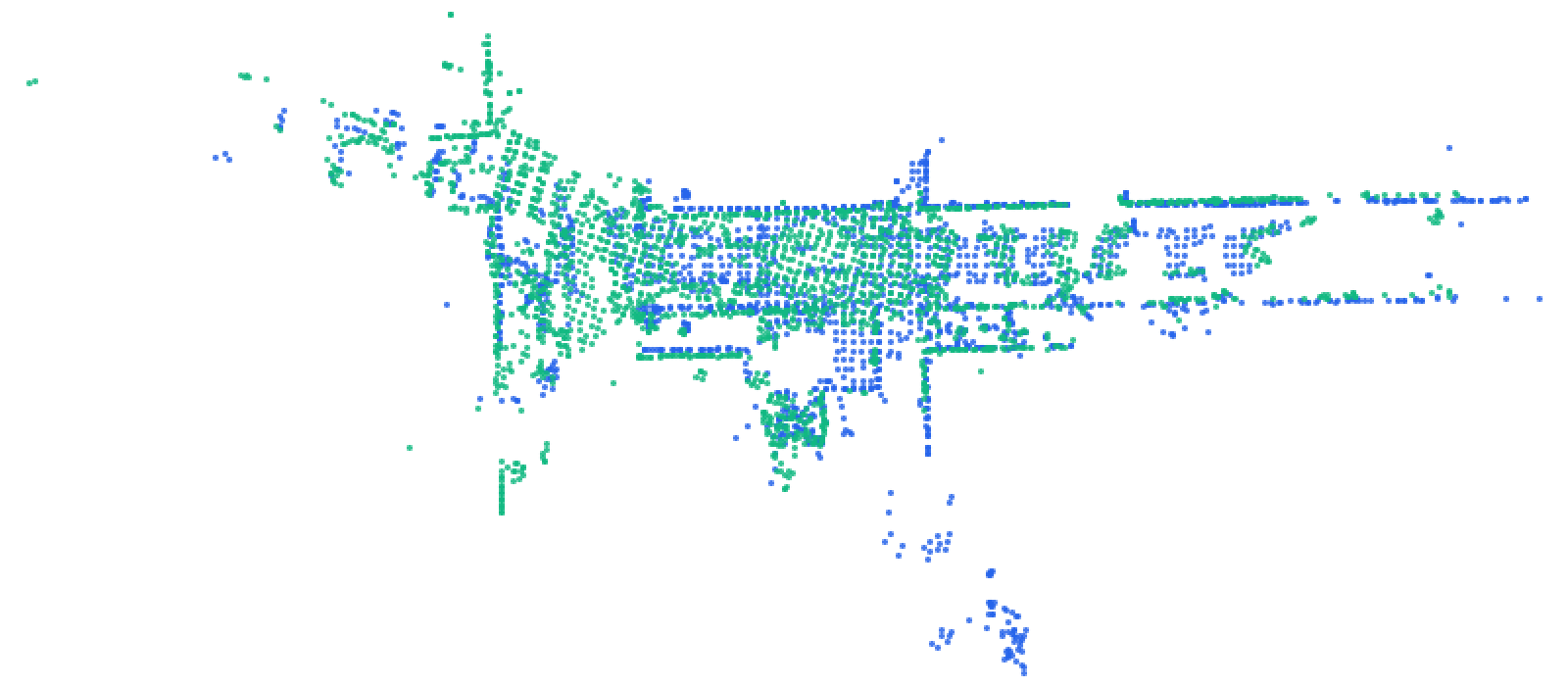}
    \caption{TEASER + Reasoning}
    \label{fig:reg_success}
\end{subfigure}

\caption{Qualitative comparison of global registration results. 
(a) TEASER fails under sparse LiDAR input. 
(b) TEASER with the proposed reasoning succeeds.}
\label{fig:global_registration}
\end{figure}

\begin{table}[t]
\centering
\caption{Global registration performance under different LiDAR resolutions. We report success rates, with relative improvements from Raw in parentheses.}
\label{tab:global_reg_channel}
\footnotesize
\setlength{\tabcolsep}{3.5pt}
\renewcommand{\arraystretch}{1.08}
\setlength{\fboxsep}{1pt}
\begin{tabular}{lccc}
\toprule
\textbf{Input} & \textbf{16 ch.} & \textbf{32 ch.} & \textbf{64 ch.} \\
\midrule
Raw    
& 79.00\% 
& 90.00\% 
& 97.00\% \\
1-step 
& \colorbox{blue!18}{\textbf{83.00\%} ($\uparrow$5.1\%)} 
& \colorbox{blue!25}{\textbf{96.00\%} ($\uparrow$6.7\%)} 
& 97.00\% ($-$0.0\%) \\
2-step 
& \colorbox{blue!6}{80.00\% ($\uparrow$1.3\%)} 
& \colorbox{blue!15}{94.00\% ($\uparrow$4.4\%)} 
& \colorbox{blue!8}{\textbf{98.00\%} ($\uparrow$1.0\%)} \\
\bottomrule
\end{tabular}
\end{table}

\noindent\textbf{Global registration.}
Table~\ref{tab:global_reg_channel} shows consistent success-rate improvements, with the largest gain at 32 channels ($+6.7\%$ at 1-step). Most gains are obtained in the first reasoning round, while later updates show diminishing returns except at 64 channels, which reaches $98\%$. Figure~\ref{fig:global_registration} shows a case where TEASER fails on raw scans but succeeds with reasoning.

\section{Limitations}
Although the proposed framework shows promising results, it has several limitations that require further investigation.
\begin{itemize}
    \item The current densification strategy samples only from highly anisotropic points. More principled covariance-based densification remains future work;
    \item The training process is time-consuming, taking about 15 minutes per epoch in our setup. Pretrained geometry reasoning models could reduce this cost; and
    \item Reciprocal updates do not always improve all metrics, as shown in Table~\ref{tab:global_reg_channel}. A reliable stopping criterion remains an open issue.
\end{itemize}

\section{Conclusion}
We presented a self-supervised geometry reasoning framework that learns underlying geometry for LiDAR SLAM without ground-truth poses or dense geometric labels. While the results demonstrate the potential of learned local geometry for LiDAR SLAM, addressing the above limitations is necessary to make the framework broadly applicable.

\bibliographystyle{plainnat}
\bibliography{references}

\end{document}